\documentclass{article} 

\usepackage[left=1.25in,top=1in,right=1.25in,bottom=1in,nohead]{geometry}
\usepackage{graphicx}
\usepackage{amsfonts,amssymb,amsmath,amsthm}
\usepackage[numbers,square]{natbib}
\usepackage{algorithm,algorithmic}
\usepackage{hyperref}
\usepackage{pdfsync}
\newtheorem{theorem}{Theorem}[section]
\newtheorem{definition}[theorem]{Definition}
\newtheorem{lemma}[theorem]{Lemma}
\newtheorem{corollary}[theorem]{Corollary}

\def\indep{\perp\!\!\!\perp}
\newcommand{\given}{\mbox{ }\vert\mbox{ }}
\newcommand{\rref}[1]{\hyperref[#1]{\ref*{#1}}}
\usepackage{hyperref}
\newcommand{\email}[1]{\href{mailto:#1}{#1}}
\newcommand{\F}{\mathcal{F}}
\newcommand{\E}{\mathbb{E}}
\renewcommand{\P}{\mathbb{P}}

\newcommand{\X}{\mathcal{X}}

\newcommand{\bsp}{\boldsymbol{\phi}}
\newcommand{\TrueMeasure}{\nu}
\newcommand{\RademacherVariable}{Z}
\newcommand{\RademacherComplexity}{\mathfrak{G}}
\DeclareMathOperator*{\argmin}{argmin} 

\title{Generalization error bounds for stationary autoregressive models}

\author{
Daniel J.~McDonald\\
Department of Statistics\\
Carnegie Mellon University\\
Pittsburgh, PA 15213 \\
\email{danielmc@stat.cmu.edu} \\
\and
Cosma Rohilla Shalizi \\
Department of Statistics \\
Carnegie Mellon University \\
Pittsburgh, PA 15213\\
\email{cshalizi@stat.cmu.edu} \\
\and
Mark Schervish \\
Department of Statistics \\
Carnegie Mellon University \\
Pittsburgh, PA 15213 \\
\email{mark@cmu.edu} \\
}

\date{Version: June 2, 2011}

\begin{document}

\maketitle

\begin{abstract}
  We derive generalization error bounds for stationary univariate
  autoregressive (AR) models.  We show that imposing stationarity
  is enough to control the Gaussian complexity
  without further regularization. This lets us use structural risk
  minimization for model selection. We
  demonstrate our methods by predicting interest rate movements. 
\end{abstract}

\section{Introduction}
\label{sec:introduction}

In standard machine learning situations, we observe one variable, $X$, and
wish to predict another variable, $Y$, with an unknown joint
distribution.  Time series models are slightly different: we observe 
a sequence of observations $\mathbf{X}_1^n\equiv\{X_t\}_{t=1}^n$ from some 
process, and we wish to predict $X_{n+h}$, for some $h\in
\mathbb{N}$. Throughout what follows,
$\mathbf{X}=\{X_t\}_{t=-\infty}^\infty$ will be a 
sequence of random variables, i.e., each $X_t$ is a measurable mapping from
some probability space $(\Omega, \mathcal{F}, \mathbb{P})$ into a measurable
space $\X$.  A block of the random sequence will be written $\mathbf{X}_i^j
\equiv \{X_t\}_{t=i}^j$, where either limit may go to infinity. 

The goal in building a predictive model is to learn a function
$\widehat{f}$ which maps the past into predictions for the future, evaluating the resulting
forecasts through a loss function $\ell(X_{n+h},\widehat{f}(\mathbf{X}_1^n))$ which gives the
cost of errors.  Ideally, we would use $f^*$, the function which
minimizes the risk
\[
R(f) \equiv \E[\ell(X_{n+h},f (\mathbf{X}_1^n))],
\]
over all $f \in \mathcal{F}$, the class of prediction functions we can use.

Since the true joint distribution of the sequence is unknown, so is $R(f)$, but it is often estimated with
the error on a training sample of size $n$
\begin{equation}
\label{eq:one}
\widehat{R}_n(f) \equiv \frac{1}{n}\sum_{t=1}^n{\ell(X_{i+h},f(\mathbf{X}_1^t))},
\end{equation}
with $\widehat{f}$ being
the minimizer of $\widehat{R}_n$ over $\mathcal{F}$. 
This is ``empirical risk minimization''. 

While $\widehat{R}_n(\widehat{f})$ converges to $R(\widehat{f})$ for many
algorithms, one can show that when $\widehat{f}$ minimizes (\ref{eq:one}),
$\E[\widehat{R}_n(\widehat{f})]\leq R(\widehat{f})$.  This is because the
choice of $\widehat{f}$ adapts to the training data, causing the training error
to be an over-optimistic estimate of the true risk. Also, training error must
shrink as model complexity grows. Thus, empirical risk minimization gives
unsatisfying results: it will tend to overfit the data and give poor
out-of-sample predictions.  Statistics and machine learning propose two
mitigation strategies.  The first is to restrict the class $\mathcal{F}$.  The
second, which we follow, is to change the optimization problem, penalizing
model complexity.  Without the true distribution, the
prediction risk or generalization error are inaccessible.  Instead, the goal is
finding bounds on the risk which hold with high probability --- ``probably
approximately correct'' (PAC) bounds. A typical result is a confidence bound on
the risk which says that with probability at least $1-\eta$,
\[
R(\widehat{f}) \leq \widehat{R}_n(\widehat{f}) + \delta(C(\mathcal{F}), n, \eta),
\]
where $C(\cdot)$ measures the complexity of the model class $\mathcal{F}$, and
$\delta(\cdot)$ is a function of this complexity, the confidence level, and the
number of observed data points.

The statistics and machine learning literature contains many generalization
error bounds for both classification and regression problems with IID data, but
their extension to time series prediction is a fairly recent development; in
1997, \citet{Vidyasagar1997} named extending such results to time series as an
important open problem.  \citet{Yu1994} sets forth many of the uniform ergodic
theorems that are needed to derive generalization error bounds for stochastic
processes.  \citet{Meir2000} is one of the first papers to construct risk
bounds for time series.  His approach was to consider a stationary but
infinite-memory process, and to decompose the training error of a predictor
with finite memory, chosen through empirical risk minimization, into three
parts:
\begin{align*}
  \widehat{R}(\widehat{f}_{p,n,d}) &= (\widehat{R}(\widehat{f}_{p,n,d}) - \widehat{R}(f^*_{p,n})) +
  (\widehat{R}(f^*_{p,n}) - \widehat{R}(f^*_p)) + \widehat{R}(f^*_p)
\end{align*}
where $\widehat{f}_{p,n,d}$ is an empirical estimate based on finite data of
length $n$, finite memory of length $p$, and complexity indexed by $d$;
$f^*_{p,d}$ is the oracle with finite memory and given complexity, and $f^*_p$
is the oracle with finite memory over all possible complexities. The three
terms amount to an estimation error incurred from the use of limited and noisy
data, an approximation error due to selecting a predictor from a class of
limited complexity, and a loss from approximating an infinite memory process
with a finite memory process.

More recently, others have provided PAC results for non-IID
data. \citet{SteinwartChristmann2009} prove an oracle inequality for generic
regularized empirical risk minimization algorithms learning from
$\alpha$-mixing processes, a fairly general sort of weak serial dependence,
getting learning rates for least-squares support vector machines (SVMs) close
to the optimal IID rates.  \citet{MohriRostamizadeh2010} prove stability-based
generalization bounds when the data are stationary and $\varphi$-mixing or
$\beta$-mixing, strictly generalizing IID results and applying to all
stable learning algorithms. (We define $\beta$-mixing below.)
\citet{KarandikarVidyasagar2009} show that if an algorithm is ``sub-additive''
and yields a predictor whose risk can be upper bounded when the data are IID,
then the same algorithm yields predictors whose risk can be bounded if data are
$\beta$-mixing. They use this result to derive generalization error bounds in
terms of the learning rates for IID data and the $\beta$-mixing coefficients.

All these generalization bounds for dependent data rely on notions of
complexity which, while common in machine learning, are hard to apply to models
and algorithms ubiquitous in the time series literature. SVMs, neural networks,
and kernel methods have known complexities, so their risk can be bounded on
dependent data as well. On the other hand, autoregressive moving average (ARMA)
models, generalized autoregressive conditional heteroskedasticity (GARCH)
models, and state-space models in general have unknown complexity and are
therefore neglected theoretically.  (This does not keep them from being used in
applied statistics, or even in machine learning and robotics, e.g.,
\citep{RuizVallejos2005,OlssonHansen2006,sak2006minimum,becker2008autoregressive,li2008forecasting}.)
Arbitrarily regularizing such models will not do, as often the only assumption
applied researchers are willing to make is that the time series is stationary.

We show that the assumption of stationarity regularizes autoregressive (AR)
models implicitly, allowing for the application of risk bounds without the need
for additional penalties. This result follows from work in the optimal control
and systems design literatures but the application is novel. In
\S\ref{sec:preliminaries}, we introduce concepts from time series and complexity
theory necessary for our results. Section \ref{sec:results} uses these results to
calculate explicit risk bounds for autoregressive models. Section
\ref{sec:application} illustrates the applicability of our methods by forecasting interest
rate movements. We discuss our results and articulate directions for future
research in \S\ref{sec:discussion}.

\section{Preliminaries}
\label{sec:preliminaries}

Before developing our results, we need to explain the idea of effective sample
size for dependent data, and the closely related measure of serial dependence
called $\beta$-mixing, as well as the Gaussian complexity technique for measuring model
complexity.

\subsection{Time series}
\label{sec:dependent-data}

Because time-series data are dependent, the number of data points $n$ in a
sample $\mathbf{X}_1^n$ exaggerates how much information the sample contains.
Knowing the past allows forecasters to predict future data (at least to some
degree), so actually observing those future data points gives less information
about the underlying data generating process than in the IID case. Thus, the
sample size term in a probabilistic risk bound must be adjusted to reflect the
dependence in the data source.  This effective sample size may be
much less than $n$.

We investigate only stationary $\beta$-mixing
input data. We first remind the reader of the notion of (strict or strong)
stationarity.
\begin{definition}[Stationarity]\label{def:stationary}
  A sequence of random variables $\mathbf{X}$ is stationary
  when all its finite-dimensional distributions are invariant over time: for
  all $t$ and all non-negative integers $i$ and $j$, the random vectors
  $\mathbf{X}_t^{t+i}$ and $\mathbf{X}_{t+j}^{t+i+j}$ have the same
  distribution.
\end{definition}
From among all the stationary processes, we restrict ourselves to ones
where widely-separated observations are asymptotically independent.
Stationarity does not imply that the random
variables $X_t$ are independent 
across time $t$, only that the distribution of $X_t$ is constant in
time. The next definition describes the nature of the serial dependence which
we are willing to allow.
\begin{definition}[$\beta$-Mixing]
  \label{defn:beta-mix}
  Let $\sigma_{i}^j=\sigma(\mathbf{X}_i^j)$ be the $\sigma$-field of events
  generated by the appropriate collection of random variables.  Let $\P_t$ be
  the restriction of $\P$ to $\sigma_{-\infty}^t$, $\P_{t+m}$ be the
  restriction of $\P$ to $\sigma_{t+m}^\infty$, and $\P_{t\otimes t+m}$ be the
  restriction of $\P$ to
  $\sigma(\mathbf{X}_{-\infty}^t,\mathbf{X}_{t+m}^\infty)$.  The {\em
    coefficient of absolute regularity}, or {\em $\beta$-mixing coefficient},
  $\beta(m)$, is given by
\begin{equation}
  \label{eq:three}
  \beta(m) \equiv ||\P_t \times \P_{t+m} - \P_{t \otimes
    t+m}||_{TV}, 
\end{equation}
where $|| \cdot ||_{TV}$ is the total variation norm. A stochastic process is
{\em absolutely regular}, or {\em $\beta$-mixing}, if $\beta(m) \rightarrow 0$
as $m\rightarrow\infty$.
\end{definition}
This is only one of many equivalent characterizations of $\beta$-mixing (see
\citet{Bradley2005} for others). This definition makes clear that a
process is $\beta$-mixing if the joint probability of events which are widely
separated in time increasingly approaches the product of the individual
probabilities, i.e., that $\mathbf{X}$ is asymptotically
independent. Typically, a supremum over $t$ is taken in
(\ref{eq:three}), however, this is unnecessary
for stationary processes, i.e.~$\beta(m)$
as defined above is independent of $t$.

\subsection{Gaussian complexity}
\label{sec:radem-compl}

Statistical learning theory provides several ways of measuring the complexity
of a class of predictive models.  The results we are using here rely
on Gaussian complexity (see, e.g.,
\citet{BartlettMendelson2002}), which can be thought of
as measuring how well the model can (seem to) fit white noise.
\begin{definition}[Gaussian Complexity]
  Let 
  $\mathbf{X}_1^n$ be a (not necessarily IID) sample drawn
  according to $\nu$. The {\em 
    empirical Gaussian complexity} is 
  \[
  \widehat{\RademacherComplexity}_n(\F) \equiv 2\E_{\RademacherVariable}\left[
    \sup_{f\in \F}\left|\frac{1}{n}\sum_{i=1}^{n}{\RademacherVariable_i f(\mathbf{X}_1^i)
      } \right|\given \mathbf{X}_1^n \right],
  \]
  where $\RademacherVariable_i$ are a sequence of random variables, independent
  of each other and everything else, and drawn from a standard Gaussian
  distribution. The {\em Gaussian complexity} is
  \[
  \RademacherComplexity_n(\F) \equiv \E_\TrueMeasure\left[
    \widehat{\RademacherComplexity}_n(\F)\right]
  \]
  where the expectation is over sample paths $D_n$ generated by $\TrueMeasure$.
\end{definition}
The term inside the supremum, $\left|\frac{1}{n}
  \sum_{i=1}^{n}{\RademacherVariable_i f(\mathbf{X}_1^i) } \right|$, is the sample
covariance between the noise $\RademacherVariable$ and the predictions of a
particular model $f$.  The Gaussian complexity takes the largest value of
this sample covariance over all models in the class (mimicking empirical risk
minimization), then averages over realizations of the noise.

Intuitively, Gaussian complexity measures how well our models could seem to
fit outcomes which were really just noise, giving a baseline against which to
assess the risk of over-fitting or failing to generalize.  As the sample size
$n$ grows, for any given $f$ the sample covariance $\left|\frac{1}{n}
  \sum_{i=1}^{n}{\RademacherVariable_i f(\mathbf{X}_1^i) } \right| \rightarrow 0$, by the
ergodic theorem; the overall Gaussian complexity should also shrink, though
more slowly, unless the model class is so flexible that it can fit absolutely
anything, in which case one can conclude nothing about how well it will predict
in the future from the fact that it performed well in the past.

\subsection{Error bounds for $\beta$-mixing data}
\label{sec:gener-error-bounds}

\citet{MohriRostamizadeh2009} present Gaussian\footnote{In fact, they
  present the bounds in terms of the 
  Rademacher complexity, a closely related idea. However, using
  Gaussian complexity instead requires no modifications to their
  results while simplifying the proofs contained here. The constant
  $(\pi/2)^{1/2}$ in Theorem~\ref{thm:mohri} is given
  in~\citet{LedouxTalagrand1991}.}
complexity-based error bounds for stationary $\beta$-mixing
sequences, a
generalization of similar bounds presented earlier for the IID case.  The results
are data-dependent and measure the complexity of a class of hypotheses based on
the training sample. 
\begin{theorem}
  \label{thm:mohri}
  Let $\mathcal{F}$ be a space of candidate predictors and let
  $\mathcal{H}$ be the 
  space of induced losses: 
  \begin{equation*}
    \mathcal{H} = \left\{ h= \ell(\cdot,f(\cdot)) : f \in \F\right\}
  \end{equation*}
  for some loss function $0 \leq \ell(\cdot,\cdot) \leq M$.  Then for
  any sample $\mathbf{X}_1^n$ drawn from a stationary 
  $\beta$-mixing distribution, and for any $\mu,m>0$ with $2\mu m=n$ and $\eta>
  4(\mu-1)\beta(m)$ where $\beta(m)$ is the mixing coefficient, with
  probability at least $1-\eta$,
  \begin{equation*}
    R(\widehat{f}) \leq \widehat{R}_n(\widehat{f}) +
    \left(\frac{\pi}{2}\right)^{1/2}\widehat{\RademacherComplexity}_{\mu}(\mathcal{H})
    + 3M\sqrt{\frac{\ln 4/\eta'}{2\mu}},
  \end{equation*}
  and
  \begin{equation*}
    R(\widehat{f}) \leq \widehat{R}_n(\widehat{f}) +
    \left(\frac{\pi}{2}\right)^{1/2}\RademacherComplexity_{\mu}(\mathcal{H})
    + M\sqrt{\frac{\ln 2/\eta'}{2\mu}},
  \end{equation*}
  where $\eta' = \eta - 4(\mu-1)\beta(m)$ in the first case or $\eta'
  = \eta - 2(\mu-1)\beta(m)$ in the second. .
\end{theorem}

The generalization error bounds in Theorem~\ref{thm:mohri} have a
straightforward interpretation. The risk of a chosen model is
controlled, with high probability, by three terms. The first term, the
training error,
describes how well the model performs in-sample. More complicated
models can more closely fit any data set, so increased complexity
leads to smaller training error. This is penalized by the second term,
the Gaussian complexity. The first bound uses the empirical
Gaussian complexity which is calculated from the data $\mathbf{X}_1^n$ while
 the second uses the expected Gaussian complexity, and is therefore
tighter.  The third term is the confidence term and
is a function only of the confidence level $\eta$ and the effective
number of data points on which the model was based $\mu$. While it was
actually trained on $n$ data points, because of dependence, this
number must be reduced. This process is accomplished by taking $\mu$ widely
spaced blocks of points. Under the asymptotic independence quantified
by $\beta$, this spacing lets us treat these blocks as independent.

\section{Results}
\label{sec:results}

Autoregressive models are used frequently in economics, finance,
and other disciplines. Their main utility lies in their straightforward
parametric form, as well as their interpretability: predictions for the future
are linear combinations of some fixed length of previous observations.  See
\citet{ShumwayStoffer2000} for a standard introduction.

Suppose that $\mathbf{X}$ is a real-valued random sequence, evolving as
\[
  X_t = \sum_{i=1}^{p} \phi_i X_{t-i} + \epsilon_t,
\]
where $\epsilon_t$ has mean zero, finite variance, $\epsilon_j \indep
\epsilon_i$ for all $i\neq j$, and $\epsilon_i \indep X_j$ for all $i>j$. This
is the traditional specification of an \emph{autoregressive order $p$} or
AR$(p)$ model. Having observed data $\{X_t\}_{t=1}^n$, and supposing $p$ to be
known, fitting the model amounts to estimating the coefficients $\{\phi\}_{i=1}^p$. The most
natural way to do this is to use ordinary least squares (OLS). Let
\begin{align*}
 \boldsymbol{\phi} &= \begin{pmatrix}\phi_1 \\ \phi_2 \\ \vdots \\
    \phi_p \end{pmatrix} &
  \mathbb{Y} &= \begin{pmatrix}  X_{p+1} \\ X_{p+2}\\ \vdots\\ X_{n-1} \\
    X_{n} \end{pmatrix} &
 \mathbb {X} &= \begin{pmatrix}
    X_p & X_{p-1} & \cdots & X_1\\X_{p+1} & X_p & \cdots & X_2 \\
    \vdots&\vdots&\ddots&\vdots\\ X_{n-2} & X_{n-3} & \cdots & X_{n-p-1}\\ X_{n-1} &
    X_{n-2} & \cdots & X_{n-p}
  \end{pmatrix}.
\end{align*}

Generalization error bounds for these processes follow from an
ability to characterize their Gaussian complexity. The theorem below
uses stationarity to bound the risk of AR models. The
remainder of this section provides the components necessary to prove the results.

\begin{theorem}
  \label{thm:main}
  Let $D_n$ be a sample of length $n$ from a stationary $\beta$-mixing
  distribution.  For any $\mu,m>0$ with $2\mu m=n$ and $\eta>4(\mu-1)\beta(m)$,
  then under squared error loss truncated at $M$, the prediction error of an
  AR$(p)$ ($p>1$) model can be bounded with probability at least $1-\eta$ using
  \begin{align*}
    R(\widehat{f}) &\leq \widehat{R}_n(\widehat{f}) + 
    \frac{4\sqrt{\pi M\log (p +1)}}{\mu}\max_{1\leq j,j' \leq
      p+1}\left(\sum_{i \in\mathcal{I}} \left\langle\mathbb{X}_i,\bsp_j - \bsp_{j'}\right\rangle^2\right)^{1/2} 
    + 3M\sqrt{\frac{\ln 4/\eta'}{2\mu}},
  \end{align*}
  or
  \begin{align*}
    R(\widehat{f}) &\leq \widehat{R}_n(\widehat{f}) + 
    \frac{4\sqrt{\pi M\log (p +1)}}{\mu}\E \left[\max_{1\leq j,j' \leq
      p+1}\left(\sum_{i \in\mathcal{I}} \left\langle\mathbb{X}_i,\bsp_j -
        \bsp_{j'}\right\rangle^2\right)^{1/2}  \right]
   + M\sqrt{\frac{\ln 2/\eta'}{2\mu}},
  \end{align*}
  where $\mathcal{I} = \{i:i=\lfloor a/2 \rfloor + 2ak, 0\leq
  k\leq\mu\}$, $\bsp_j$ is the $j^{th}$ vertex of the stability
  domain, and $\mathbb{X}_i$ is the $i^{th}$ row of the design matrix $\mathbb{X}$.
\end{theorem}


For $p=1$ slight adjustments
are required. We state this result as a corollary.
\begin{corollary}
  \label{cor:1}
  Under the same conditions as above, the prediction error of an
  AR$(1)$ model can be bounded with probability at least $1-\eta$
  using
\begin{align*}
    R(\widehat{f}) &\leq \widehat{R}_n(\widehat{f}) + 
    \frac{4}{\mu}\sqrt{\frac{M}{2}} \left( \sum_{i \in \mathcal{I}} X_i^2\right)^{1/2} + 3M\sqrt{\frac{\ln 4/\eta'}{2\mu}},
  \end{align*}
  or
  \begin{align*}
    R(\widehat{f}) &\leq \widehat{R}_n(\widehat{f}) + 
    \frac{4}{\mu}\sqrt{\frac{M}{2}}\E \left[\left(\sum_{i \in
        \mathcal{I}}\mathbb{X}_i^2\right)^{1/2}\right]+ M\sqrt{\frac{\ln 2/\eta'}{2\mu}}.
  \end{align*}
\end{corollary}

\subsection{Proof components}

To prove Theorem~\ref{thm:main} it is necessary
to control the size of 
the model class by using the stationarity assumption.

\subsubsection{Stationarity controls the hypothesis space}
\label{sec:stationary-ar-models}

Define, as an estimator of $\bsp$,
\begin{equation}
  \label{eq:2}
  \widehat{\bsp} \equiv \argmin_{\bsp}
  ||\mathbb{Y} - \mathbb{X}\bsp||_2^2,
\end{equation}
where $||\cdot||_2$ is the Euclidean norm.\footnote{There are other ways to estimate
AR models, but they typically amount to very similar optimization problems.}
Equation~\ref{eq:2} has the usual closed form OLS solution:
\begin{equation}
  \label{eq:3}
  \widehat{\bsp} = \mathbb{(X'X)}^{-1}\mathbb{X'Y}.
\end{equation}

Despite the simplicity of Eq.\ \ref{eq:3}, modellers often require that the
estimated autoregressive process be stationary. This can be checked
algebraically: the complex roots of the polynomial
\[
 Q_p(z) = z^p - \phi_1 z^{p-1} - \phi_2 z^{p-2} - \cdots - \phi_p
\]
must lie strictly inside the unit circle.  Eq.\ \ref{eq:2} is thus not quite
right for estimating a stationary autoregressive model, as it does not
incorporate this constraint.

Constraining the roots of $Q_p(z)$ constrains the coefficients
$\boldsymbol{\phi}$.  The set $\boldsymbol{\phi}$ where the process is
stationary is the stability domain, $\mathcal{B}_p$. Clearly, $\mathcal{B}_1$
is just $|\phi_1|<1$. \citet{FamMeditch1978} gives a recursive method for
determining $\mathcal{B}_p$ for general $p$. In particular, they show that the
convex hull of the space of stationary solutions is a convex polyhedron with
vertices at the extremes of the $\mathcal{B}_p$.  This convex
hull basically determines the complexity of stationary AR
models.

\subsubsection{Gaussian complexity of AR models}
\label{sec:radem-compl-ar}

Returning to the AR($p$) model, it is necessary to find the Gaussian complexity of the
function class
\[
  \F_p = \left\{\boldsymbol{\phi} : x_t = \sum_{i=1}^p
    \phi_i x_{t-i} \mbox{ and } x_t \mbox{ is stationary}\right\}.
\]

\begin{theorem}
  \label{thm:gauss}
  For the AR($p$) model with $p>1$, the empirical Gaussian complexity
  is given by
  \begin{equation*}
    \widehat{\mathfrak{G}}_k(\F) \leq \frac{2\sqrt{2}}{n}(\log (p+1))^{1/2}
    \max_{1\leq j,j' \leq p+1} \left( \sum_{i=1}^k \left\langle \mathbb{X}_i, \bsp_j -
      \bsp_{j'}\right\rangle^2 \right)^{1/2},
  \end{equation*}
  where $\bsp_j$ is the $j^{th}$ vertex of the stability
  domain and $\mathbb{X}_i$ is the $i^{th}$ row of the design matrix $\mathbb{X}$.
\end{theorem}
The proof relies on the following version of Slepian's Lemma (see, for
example~\citet{LedouxTalagrand1991} or~\citet{BartlettMendelson2002}).
\begin{lemma}[Slepian]
  \label{lem:slepian}
  Let $V_1,\ldots,V_k$ be random variables such that for all $1\leq
  j\leq k$, $V_j = \sum_{i=1}^n a_{ij}g_i$ where $g_1,\ldots,g_n$ are
  iid standard normal random variables. Then,
  \begin{equation*}
    \E \big[\max_{j} V_j\big] \leq \sqrt{2} (\log k)^{1/2} \max_{j,j'}
    \sqrt{\E\left[(V_j - V_{j'})^2\right]}.
  \end{equation*}
\end{lemma}

\begin{proof}
  [Proof of Theorem~\ref{thm:gauss}]
  \begin{align*}
    \widehat{\mathfrak{G}}_n{\F} &= \E \sup_{f\in\F}
    \frac{2}{n}\sum_{i=1}^n g_i f(x_i) = \E \sup_{\bsp \in \mathcal{B}_p}
    \frac{2}{n}\sum_{i=1}^n g_i \langle \mathbb{X}_i,\bsp \rangle \\
    &= \E \sup_{\bsp \in \mathcal{B}_p}
    \left\langle \frac{2}{n}\sum_{i=1}^n g_i \mathbb{X}_i, \bsp\right\rangle
    =\E \sup_{\bsp \in conv(\mathcal{B}_p)}
    \left\langle \frac{2}{n}\sum_{i=1}^n g_i \mathbb{X}_i,\bsp\right\rangle,
  \end{align*}
  where the last equality follows from Theorem 12
  in~\citep{BartlettMendelson2002}. By standard results from convex
  optimization, this supremum is attained at one of the vertices of
  $conv(\mathcal{B}_p)$. Therefore,
  \begin{equation*}
   \widehat{\mathfrak{G}}_n(\F) = \E \big[\max_j \frac{2}{n} \sum_{i=1}^n
    g_i \langle\mathbb{X}_{i},\bsp_j\rangle\big],
  \end{equation*}
  where $\bsp_j$ is the $j^{th}$ vertex of $conv(\mathcal{B}_p)$. Let $V_j
  = \sum_{i=1}^n g_i \langle\mathbb{X}_{i},\bsp_j\rangle$. Then by the Lemma~\ref{lem:slepian},
  \begin{align*}
   \widehat{\mathfrak{G}}_n(\F) &\leq \frac{2\sqrt{2}}{n}(\log
    p+1)^{1/2} \max_{j,j'} \sqrt{\E[V_j - V_{j'}]^2} \\
   &= \frac{2\sqrt{2}}{n}(\log
    p+1)^{1/2} \max_{j,j'} \sqrt{\E \left[\sum_{i=1}^n g_i\left\langle\mathbb{X}_{i},
        \bsp_j-\bsp_{j'}\right\rangle\right]^2} \\
   &=\frac{2\sqrt{2}}{n}(\log p+1)^{1/2} \max_{1\leq j,j' \leq p}\sqrt{
    \sum_{i=1}^n \left\langle\mathbb{X}_{i},
        \bsp_j-\bsp_{j'}\right\rangle^2 } 
  \end{align*}
where $\mathbb{X}_{i}$ is the $i^{th}$-entry of the design matrix.
\end{proof}

When $p=1$, as in Corollary~\ref{cor:1}, we can calculate the complexity
directly. The proof's last line shows that we are essentially interested in the
diameter of the stability domain $\mathcal{B}_p$ projected onto the column
space of $\mathbb{X}$, which gives a tighter bound than that from the general
results on linear prediction in e.g.~\citet{KakadeSridharan2008}.

Since we care about the complexity of the model class $\F$ viewed through the
loss function $\ell$, we must also account for this additional complexity. For
$c$-Lipschitz loss functions, this just means multiplying $\mathfrak{G}_n(\F)$
by $2c$.

\section{Application}
\label{sec:application}

We illustrate our results by predicting interest rate changes --- specifically,
the 10-year Treasury Constant Maturity Rate series from the Federal Reserve
Bank of St. Louis' FRED database\footnote{Available at
  \url{http://research.stlouisfed.org/fred2/series/DGS10?cid=115}.}
--- recorded
daily from January 2, 1962 to August 31, 2010.  Transforming the series into
daily natural-log growth rates leaves $n=12150$ observations (Figure
\ref{fig:one}).
\begin{figure*}[t]
  \centering
  \includegraphics[width=5in]{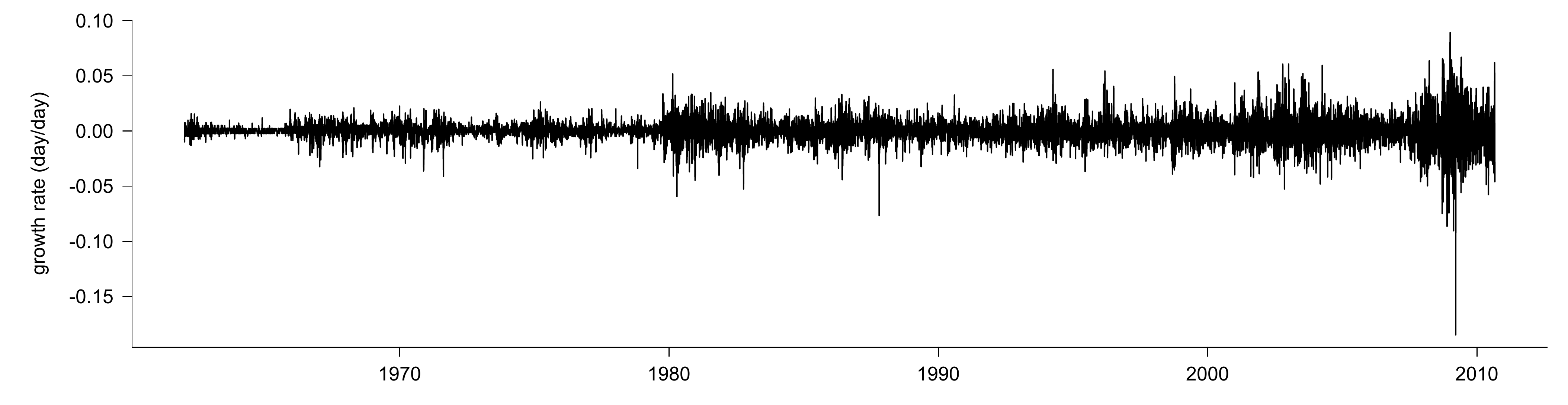}
  \caption{Growth rate of 10-year treasury bond}
  \label{fig:one}
\end{figure*}
The changing variance apparent in the figure is why interest rates are
typically forecast with GARCH(1,1) models.  For this illustration however, we
will use an AR$(p)$ model, picking the memory order $p$ by the risk bound.

Figure \ref{fig:two} shows the training error
\[
\widehat{R}_n(\widehat{f}) = \frac{1}{n-p}\sum_{t=p+1}^n (\widehat{X}_t - X_t)^2
\]
where $X_t$ is the $t^{th}$ datapoint, and $\widehat{X}_t$ is the model's
prediction. $\widehat{R}_n$ shrinks as the order of the model ($p$) grows, as
it must since ordinary least squares minimizes $\widehat{R}_n$ for a given $p$.
Also shown is the gap between the AIC for different $p$ and the lowest
attainable value; this would select an AR(36) model.
\begin{figure}
  \centering
  \includegraphics[width=2.5in]{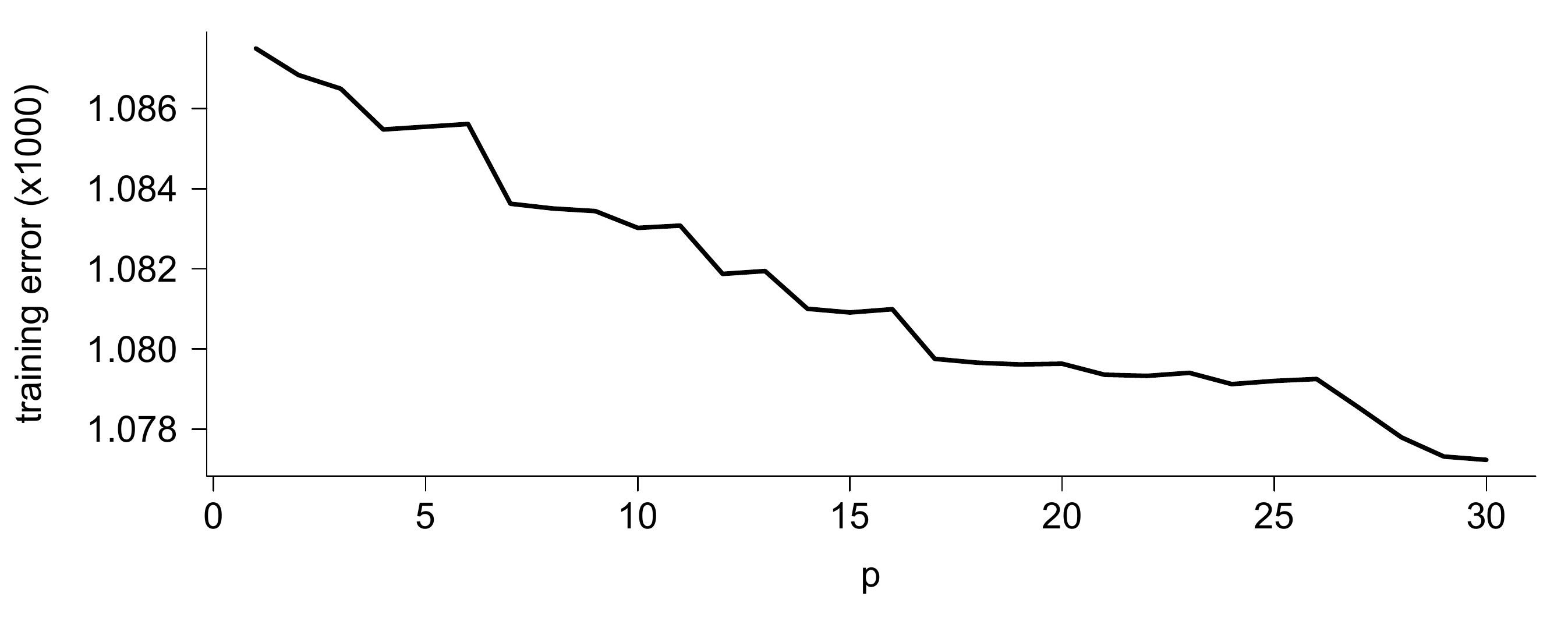}
  \includegraphics[width=2.5in]{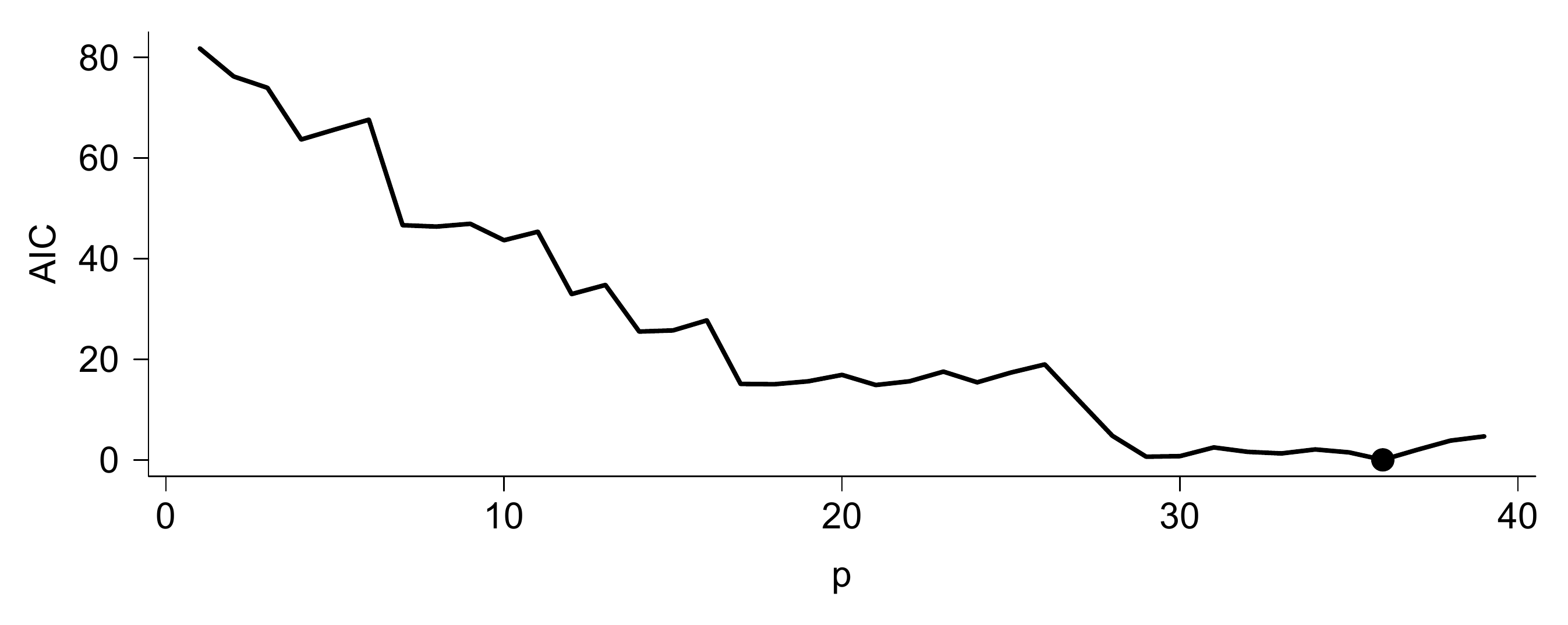}
  \caption{Training error (top panel) and AIC (bottom panel) against model
    order}
  \label{fig:two}
\end{figure}

A better strategy uses the probabilistic risk bound derived above. The goal of
model selection is to pick, with high probability, the model with the smallest risk;
this is Vapnik's structural risk minimization principle.  Here, it is clear
that AIC is dramatically overfitting. The optimal model using the risk bound is
an AR($1$). Figure \ref{fig:three} plots the risk bound against $p$ with the
loss function truncated at $0.05$. (No daily interest rate change has ever had
loss larger than $0.034$, and results are fairly insensitive to the level of
the loss cap.)  This bound says that with 95\% probability, \emph{regardless of
  the true data generating process}, the AR(1) model will make mistakes with
squared error no larger than $0.0079$. If we had instead predicted with zero,
this loss would have occurred three times.
\begin{figure}
  \centering
  \includegraphics[width=3in]{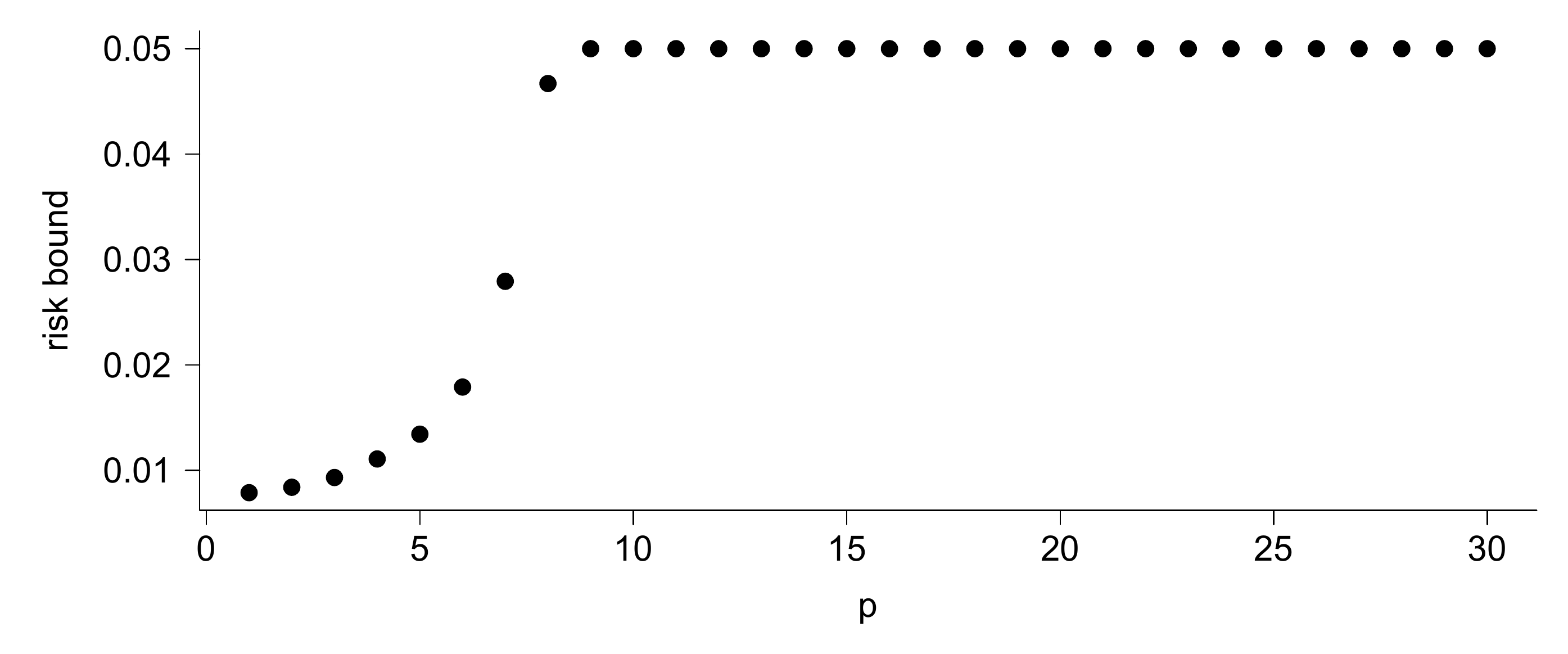}
 \caption{Generalization error bound for different model orders}
  \label{fig:three}
\end{figure}

One issue with Theorem \ref{thm:mohri} is that it requires knowledge of the
$\beta$-mixing coefficients, $\beta(m)$. Of course, the dependence structure of
this data is unknown, so we calculated it under generous assumptions on the
data generating process. In a homogeneous Markov process, the $\beta$-mixing
coefficients work out to
\[
  \beta(m) = \int \pi(dx) ||P^m(x,\cdot) - \pi||_{TV}
\]
where $P^m(x,\cdot)$ is the $m$-step transition operator and $\pi$ is the
stationary distribution \citep{Mokkadem1988,Davydov1973}. Since AR models are
Markovian, we estimated an AR$(q)$ model with Gaussian errors for $q$ large and
calculated the mixing coefficients using the stationary and transition
distributions. To create the bound, we used $m=7$ and $\mu=867$.  We address
non-parametric estimation of $\beta$-mixing coefficients elsewhere [Anon.].

\section{Discussion}
\label{sec:discussion}

We have constructed a finite-sample predictive risk bound for autoregressive
models, using the stationarity assumption to constrain OLS estimation.
Interestingly, stationarity --- a common assumption among applied researchers
--- constrains the model space enough to yield bounds without further
regularization.  Moreover, this is the first predictive risk bound we know of
for any of the standard models of time series analysis.

Traditionally, time series analysts have selected models by blending empirical
risk minimization, more-or-less quantitative inspection of the residuals (e.g.,
the Box-Ljung test; see \citep{ShumwayStoffer2000}), and AIC.  In many
applications, however, what really matters is prediction, and none of these
techniques, including AIC, controls generalization error, especially with
mis-specification.  (Cross-validation is a partial exception, but it is tricky
for time series; see \citep{Racine2000} and references therein.)  Our bound
controls prediction risk directly.  Admittedly, our bound covers only
univariate autoregressive models, the plainest of a large family of traditional
time series models, but we believe a similar result will cover the more
elaborate members of the family such as vector autoregressive,
autoregressive-moving average, or autoregressive conditionally heteroskedastic
models.  While the characterization of the stationary domain from
\citep{FamMeditch1978} on which we relied breaks down for such models, they are
all variants of the linear state space model
\citep{DurbinKoopman2001}, whose parameters are restricted under
stationarity, and so we hope to obtain a general risk bound, 
possibly with stronger variants for particular specifications.

\bibliography{AllReferences}

\begin{thebibliography}{22}
\providecommand{\natexlab}[1]{#1}
\providecommand{\url}[1]{\texttt{#1}}
\expandafter\ifx\csname urlstyle\endcsname\relax
  \providecommand{\doi}[1]{doi: #1}\else
  \providecommand{\doi}{doi: \begingroup \urlstyle{rm}\Url}\fi

\bibitem[Bartlett and Mendelson(2002)]{BartlettMendelson2002}
Peter~L. Bartlett and Shahar Mendelson.
\newblock Rademacher and gaussian complexities: Risk bounds and structural
  results.
\newblock \emph{Journal of Machine Learning Research}, 3:\penalty0 463--482,
  2002.

\bibitem[Becker et~al.(2008)Becker, Tummala, and
  Riviere]{becker2008autoregressive}
B.C. Becker, H.~Tummala, and C.N. Riviere.
\newblock Autoregressive modeling of physiological tremor under microsurgical
  conditions.
\newblock In \emph{Engineering in Medicine and Biology Society, 2008. EMBS
  2008. 30th Annual International Conference of the IEEE}, pages 1948--1951.
  IEEE, 2008.

\bibitem[Bradley(2005)]{Bradley2005}
Richard~C. Bradley.
\newblock Basic properties of strong mixing conditions. a survey and some open
  questions.
\newblock \emph{Probability Surveys}, 2:\penalty0 107--144, 2005.
\newblock URL \url{http://arxiv.org/abs/math/0511078}.

\bibitem[Davydov(1973)]{Davydov1973}
Y.A. Davydov.
\newblock Mixing conditions for markov chains.
\newblock \emph{Theory of Probability and its Applications}, 18\penalty0
  (2):\penalty0 312--328, 1973.

\bibitem[Durbin and Koopman(2001)]{DurbinKoopman2001}
J.~Durbin and S.J. Koopman.
\newblock \emph{Time Series Analysis by State Space Methods}.
\newblock Oxford Univ Press, Oxford, 2001.

\bibitem[Fam and Meditch(1978)]{FamMeditch1978}
Adly~T. Fam and James~S. Meditch.
\newblock A canonical parameter space for linear systems design.
\newblock \emph{IEEE Transactions on Automatic Control}, 23\penalty0
  (3):\penalty0 454--458, 1978.

\bibitem[Kakade et~al.(2008)Kakade, Sridharan, and Tewari]{KakadeSridharan2008}
Sham~M. Kakade, Karthik Sridharan, and Ambuj Tewari.
\newblock On the complexity of linear prediction: Risk bounds, margin bounds,
  and regularization.
\newblock Technical report, NIPS, 2008.
\newblock URL \url{http://ttic.uchicago.edu/~karthik/rad-paper.pdf}.

\bibitem[Karandikar and Vidyasagar(2009)]{KarandikarVidyasagar2009}
R.~L. Karandikar and M.~Vidyasagar.
\newblock Probably approximately correct learning with beta-mixing input
  sequences.
\newblock submitted for publication, 2009.

\bibitem[Ledoux and Talagrand(1991)]{LedouxTalagrand1991}
M.~Ledoux and M.~Talagrand.
\newblock \emph{Probability in Banach Spaces: Isoperimetry and Processes}.
\newblock A Series of Modern Surveys in Mathematics. Springer Verlag, Berlin,
  1991.
\newblock ISBN 3540520139.

\bibitem[Li and Moore(2008)]{li2008forecasting}
J.~Li and A.W. Moore.
\newblock Forecasting web page views: Methods and observations.
\newblock \emph{Journal of Machine Learning Research}, 9:\penalty0 2217--2250,
  2008.

\bibitem[Meir(2000)]{Meir2000}
Ron Meir.
\newblock Nonparametric time series prediction through adaptive model
  selection.
\newblock \emph{Machine Learning}, 39\penalty0 (1):\penalty0 5--34, 2000.
\newblock URL
  \url{http://www.ee.technion.ac.il/~rmeir/Publications/MeirTimeSeries00.pdf}.

\bibitem[Mohri and Rostamizadeh(2009)]{MohriRostamizadeh2009}
Mehryar Mohri and Afshin Rostamizadeh.
\newblock Rademacher complexity bounds for non-iid processes.
\newblock In D.~Koller, D.~Schuurmans, Y.~Bengio, and L.~Bottou, editors,
  \emph{Advances in Neural Information Processing Systems 21}, volume~21, pages
  1097--1104, 2009.

\bibitem[Mohri and Rostamizadeh(2010)]{MohriRostamizadeh2010}
Mehryar Mohri and Afshin Rostamizadeh.
\newblock Stability bounds for stationary $\varphi$-mixing and $\beta$-mixing
  processes.
\newblock \emph{Journal of Machine Learning Research}, 11:\penalty0 789--814,
  February 2010.

\bibitem[Mokkadem(1988)]{Mokkadem1988}
A.~Mokkadem.
\newblock Mixing properties of arma processes.
\newblock \emph{Stochastic processes and their applications}, 29\penalty0
  (2):\penalty0 309--315, 1988.

\bibitem[Olsson and Hansen(2006)]{OlssonHansen2006}
R.K. Olsson and L.K. Hansen.
\newblock Linear state-space models for blind source separation.
\newblock \emph{The Journal of Machine Learning Research}, 7:\penalty0
  2585--2602, 2006.
\newblock ISSN 1532-4435.

\bibitem[Racine(2000)]{Racine2000}
J.~Racine.
\newblock Consistent cross-validatory model-selection for dependent data:
  Hv-block cross-validation.
\newblock \emph{Journal of econometrics}, 99\penalty0 (1):\penalty0 39--61,
  2000.

\bibitem[Ruiz-del Solar and Vallejos(2005)]{RuizVallejos2005}
J.~Ruiz-del Solar and P.~Vallejos.
\newblock Motion detection and tracking for an aibo robot using motion
  compensation and kalman filtering.
\newblock In \emph{Lecture Notes in Computer Science 3276 (RoboCup 2004)},
  pages 619--627. Springer Verlag, 2005.

\bibitem[Sak et~al.(2006)Sak, Dowe, and Ray]{sak2006minimum}
M.~Sak, D.L. Dowe, and S.~Ray.
\newblock {Minimum message length moving average time series data mining}.
\newblock In \emph{Computational Intelligence Methods and Applications, 2005
  ICSC Congress on}, page~6. IEEE, 2006.
\newblock ISBN 1424400201.

\bibitem[Shumway and Stoffer(2000)]{ShumwayStoffer2000}
R.H. Shumway and D.S. Stoffer.
\newblock \emph{Time Series Analysis and Its Applications}.
\newblock Springer Series in Statistics. Springer Verlag, New York, 2000.

\bibitem[Steinwart and Christmann(2009)]{SteinwartChristmann2009}
Ingo Steinwart and Andreas Christmann.
\newblock Fast learning from non-i.i.d. observations.
\newblock In Y.~Bengio, D.~Schuurmans, J.~Lafferty, C.~K.~I. Williams, and
  A.~Culotta, editors, \emph{Advances in Neural Information Processing Systems
  22}, pages 1768--1776. MIT Press, 2009.
\newblock URL \url{http://books.nips.cc/papers/files/nips22/NIPS2009_1061.pdf}.

\bibitem[Vidyasagar(1997)]{Vidyasagar1997}
M.~Vidyasagar.
\newblock \emph{A Theory of Learning and Generalization: With Applications to
  Neural Networks and Control Systems}.
\newblock Springer Verlag, Berlin, 1997.

\bibitem[Yu(1994)]{Yu1994}
Bin Yu.
\newblock Rates of convergence for empirical processes of stationary mixing
  sequences.
\newblock \emph{The Annals of Probability}, 22\penalty0 (1):\penalty0 94--116,
  1994.

\end{thebibliography}
\bibliographystyle{plainnat}
\end{document}